\documentclass[11pt]{article} % For LaTeX2e
\usepackage{fullpage}

\usepackage{wrapfig}
\usepackage{comment}
\usepackage[utf8]{inputenc} % allow utf-8 input
\usepackage[T1]{fontenc}    % use 8-bit T1 fonts
\usepackage{hyperref}       % hyperlinks
\usepackage{url}            % simple URL typesetting
\usepackage{booktabs}       % professional-quality tables
\usepackage{tcolorbox}
\usepackage{nicefrac}       % compact symbols for 1/2, etc.
\usepackage{microtype}      % microtypography
\usepackage{xcolor}
\usepackage{subcaption}
\usepackage{caption}
\usepackage[normalem]{ulem}
\usepackage{tikz}
\usepackage{amsfonts}       % blackboard math symbols 
\usepackage{amssymb}
\usepackage{amsmath}
\usepackage{amsthm}
\usepackage{mathtools}
\usepackage{multicol}
\usepackage{multirow}
\usepackage{amsthm}
\newtheorem{thm}{Theorem} %[section] 

\newtheorem{prop}{Proposition}
\theoremstyle{definition}

\newtheorem{defn}{Definition}%[section]
\usepackage{color}
\usepackage{array}
\usepackage{algorithm}
\usepackage{algorithmic}
\newcounter{gaocomm}

\definecolor{blue-violet}{rgb}{0.00,0.75,0.90}
\definecolor{mygreen}{rgb}{0.0, 0.5, 0.0}
\definecolor{awesome}{rgb}{1.0, 0.13, 0.32}
\definecolor{bostonuniversityred}{rgb}{0.8, 0.0, 0.0}

\renewcommand\arraystretch{1.4}

\usepackage{soul}
\usepackage{color}
\usepackage[normalem]{ulem}
\newcounter{ToDo1}
\newcounter{guocomm}
\newcounter{Note1}
\definecolor{blue-violet1}{rgb}{0.54, 0.17, 0.89}
\definecolor{mygreen}{rgb}{0.0, 0.5, 0.0}
\definecolor{awesome}{rgb}{1.0, 0.13, 0.32}
\definecolor{wsuacdred}{rgb}{0.93, 0.0, 0.2}
\definecolor{wsucrimson}{rgb}{0.6, 0.0, 0.2}
\newcommand{\guorm}[1]{\ignorespaces}
\usepackage[utf8]{inputenc}

\title{\textbf{Design Your Own Universe: A Physics-Informed Agnostic Method for
Enhancing Graph Neural Networks}}

\author{Dai Shi \footnote{Equal contributions from the first three authors.} \footnote{University of Sydney,
(\text{dai.shi, lequan.lin, zhiyong.wang, junbin.gao@sydney.edu.au})}
\and Andi Han\footnote{andi.han@riken.jp} \footnotemark[1]
\and Lequan Lin \footnotemark[1]\footnotemark[2] 
\and  Yi Guo \footnote{Western Sydney University,
(\text{y.guo@westernsydney.edu.au}).} 
\and Zhiyong Wang \footnotemark[2] \and  Junbin Gao \footnotemark[2]
}
\date{}

\begin{document}

\maketitle

\section{Introduction}
Graph Neural Networks (GNNs) have demonstrated exceptional performance in learning tasks involving graph-structured data \cite{sperduti1993encoding,gori2005new,scarselli2008graph,kipf2016semi}. Recent studies in the GNN domain have unveiled a connection between continuous physical diffusion processes and their discretized versions on graphs \cite{chamberlain2021grand,han2023continuous}. Accordingly, numerous physics-informed methods have been proposed to mitigate issues in GNN such as over-smoothing (OSM) \cite{thorpe2022grand}, over-squashing (OSQ) \cite{topping2021understanding}, and heterophily adaption \cite{di2022graph}. Specifically, in addressing the OSM problem, one needs to ensure that node features remain distinguishable after several iterations of GNN \cite{cai2020note,rusch2023survey}. To achieve this, it may be necessary for some connected nodes to ``repulse'' each other, so they can be categorized into different classes from GNN outputs \cite{wang2022acmp,gao2022repulsion}. Methods that enable GNNs to do this naturally enhance the networks' ability to adapt to heterophily graphs, in which connected nodes are more likely to belong to different classes of labels. On the other hand, to mitigate OSQ issues, it is crucial to allow information to flow more effectively through GNN propagation \cite{shi2023exposition}. This can be accomplished by leveraging strategies such as graph adjacency rewiring and reweighting based on topological features like curvatures \cite{topping2021understanding,fesser2023mitigating,shi2023curvature,shi2024new} and spectral expanders \cite{karhadkar2023fosr,black2023understanding,banerjee2022oversquashing}. Additionally, it has been observed that there is a trade-off between the OSM and OSQ issues in GNNs, leading to a few recent studies to focus on alleviating both issues simultaneously \cite{shao2023unifying,giraldo2023trade}.

Based on the aforementioned contents, to mitigate both OSM and OSQ problems, one may prefer to leverage a mixed operation that involves graph rewiring to reduce the OSQ issue and induce repulsive forces between nodes to alleviate OSM issues. In this paper, inspired by the field of particle systems in physics, we introduce a novel model agnostic paradigm that leverages node labeling information to induce both repulsive forces and rewiring on the graph. Our approach can be applied to many GNNs for deep learning on graphs. {Specifically, our method adds the so-called collapsing nodes (CNs) with different labeling information to the original graph. These CNs are then connected to all other nodes with known labels with different signs of their edge weights, depending on whether node labels are aligned. That is, if the nodes in the original graph share the same label with the connected CN, then the edge weight will be positive; otherwise, it is negative. This operation ensures nodes (with known labels) will be ``attracted'' by the CNs that contain the same label as them and be ``repulsed''  to those CNs with different labels. In the later section, we show that this dynamic between nodes helps our model avoid the OSM problem. Another benefit of this label-based rewiring between CNs and original nodes is that this operation makes the feature information transaction ``easier'' in the newly rewiring graph, thus naturally mitigating the OSQ problem. Lastly, we also quantified the effects of incorporating CNs on the graph curvature and the filtering functions on the graph spectra, which are commonly leveraged as indicators of OSQ \cite{topping2021understanding, fesser2023mitigating} and OSM \cite{di2022graph} problems, respectively. 
}

\paragraph{Contributions} 
First, we introduce the notion of CNs that served as reliable ``gravitational'' sources via GNNs training. We show that guided by the node labeling information, the graph adjacency, expanded by the CNs, makes the propagation of GNNs through both attractive and repulsive forces. Second, we theoretically verify that our proposed models can mitigate both OSM and OSQ issues and the heterophily adaption problem. We also discuss the spectral property of the CNs rewired graph and show that it is those negative eigenvalues that enhance our models' adaption power to heterophilic graphs. Furthermore, we show the impact of CNs via graph curvature and the trade-off relation between OSM and OSQ problems. Lastly, we conduct various experiments to demonstrate the effectiveness of our method on both homophilic and heterophilic graphs as well as long-term learning tasks.   

\section{Preliminary}
\paragraph{Graph Basics and Graph Homophily}
Let $\mathcal{G} = (\mathcal{V}, \mathcal{E}_0)$ represent an connected undirected graph with $N$ nodes, where $\mathcal{V}$ and $\mathcal{E}_0$ denote the sets of nodes and edges, respectively. The adjacency matrix $\mathbf{A} \in \mathbb{R}^{N \times N}$ is defined such that $a_{i,j} = 1$ if $(i, j) \in \mathcal E_0$ and zero otherwise. We introduce $\mathbf{X} \in \mathbb{R}^{N \times d_0}$ to represent the matrix of $d_0$-dimensional nodes features, with $\mathbf{x}_i \in \mathbb{R}^{d_0}$ as its $i$-th row (transposed).
Additionally, we define $\mathbf{Y} \in \mathbb{R}^{N \times C}$ as the node label matrix, comprising label vectors for the labeled nodes (via one-hot coding) and zero vectors for the unlabeled nodes, where $C$ is the total number of classes. Apart from these basic notations on graph, we also recall the notion of so-called graph homo/heterophily, which shows how labels are distributed among connected nodes.
\begin{defn}[Homophily and Heterophily]
\label{HomophilyHeterophily}
The homophily or heterophily of a network is used to define the relationship between labels of connected nodes. Denote $\mathcal N_i\subseteq \mathcal V$ as the neighbors of node $i$. The level
of homophily of a graph is measured by the positive score $\mathcal{H(G)} = \mathbb{E}_{v_i \in \mathcal{V}}[|\{v_j: v_j\in \mathcal{N}_{i} \text{ and }  y_j= y_i\}|/|\mathcal{N}_i|]$. {A score $\mathcal{H(G)}$ close to 1}  corresponds to strong homophily
while a score $\mathcal{H(G)} $ nearly 0 indicates strong heterophily. We say that
a graph is a homophilic (heterophilic) graph if it has   stronger homophily (heterophily), or simply 
strong homophily (heterophily).  

\end{defn}
Based on the definition of graph homophily, one can see that, compared to homophily graphs which require a learning model to predict nearly identical labels to connected nodes, distinct label predictions are preferred for heterophilic graphs.

\paragraph{MPNNs and Graph Neural Diffusion}
Consider node $i$ with feature representation $\mathbf{h}_i^{(\ell)}$ at layer $\ell$, and  $\mathbf{h}^{(0)}_i = \mathbf x_i$. Message Passing Neural Networks (MPNNs) \cite{gilmer2017neural} use message functions $\psi_\ell : \mathbb R^{d_\ell} \times \mathbb R^{d_\ell} \rightarrow \mathbb R^{d_\ell'}$ and update functions $\phi_\ell: \mathbb R^{d_\ell} \times \mathbb R^{d_\ell'} \rightarrow \mathbb R^{d_{\ell +1}}$ defined by:
\begin{align}\label{mpnn}
\mathbf{h}_i^{(\ell+1)} = \phi_\ell \left (\mathbf{h}_i^{(\ell)}, \sum_{j \in \mathcal N_i} {\mathbf{A}}_{ij} \psi_\ell(\mathbf{h}_i^{(\ell)}, \mathbf{h}_j^{(\ell)})\right).
\end{align}
It is worth noting that $\mathbf A$ can be replaced with the features that are defined on the edges of the graph. For the simplicity of our analysis, we only consider MPNNs with $\mathbf A$ included. Classical GNNs like GCN \cite{kipf2016semi} and GAT \cite{velivckovic2018graph} are the examples of this MPNN framework. Crucially, MPNNs also act as solvers for discrete dynamics on graphs, like the well-known graph neural diffusion \cite{chamberlain2021grand,thorpe2022grand}. Specifically, 
\cite{chamberlain2021grand} incorporated the message passing scheme and its variants in their GRAND models as described as:
\begin{align}\label{grad}
\frac{\partial}{\partial t}\mathbf h(t) = \mathrm{div}\left(\mathbf G(\mathbf h(t),t)\nabla \mathbf h(t)\right),
\end{align}
where $\mathbf G(\mathbf h(t),t) = \mathrm{diag}(a(\mathbf h_i(t), \mathbf h_j(t),t))$ in which $a$ denotes a function that quantifies the similarity between node features, such as the attention coefficient \cite{velivckovic2018graph}. We further let $\nabla \mathbf h$ represents the graph gradient operator, defined as $\nabla : L^2(\mathcal V) \rightarrow L^2(\mathcal E_0)$  such that $(\nabla \mathbf h)_{ij} = \mathbf h_j - \mathbf h_i$, where $L^2(\mathcal V)$ and $L^2(\mathcal E_0)$ be Hilbert spaces for real-valued functions on $\mathcal V$ and $\mathcal E_0$, respectively with the inner products given by 
$$ \langle f, g \rangle_{L^2(\mathcal V)} = \sum_{i \in \mathcal V} f_i g_i, \,\,\,  \langle F, G \rangle_{L^2(\mathcal E_0)} = \sum_{(i,j) \in \mathcal E_0}  F_{ij} G_{ij},$$
for $f,g: \mathcal V \rightarrow \mathbb R$ and $F, \,\, G: \mathcal E_0 \rightarrow \mathbb R$. Similarly, we denote graph divergence ${\mathrm {div}}: L^2(\mathcal E_0) \rightarrow L^2(\mathcal V)$,  being the inverse of graph gradient operator, is defined as $(\mathrm{div}  F)_i = \sum_{j :(i,j) \in \mathcal E_0}  F_{ij}$.

{
\paragraph{OSM and Heterophily Adaption}
Although the phenomenon of OSM in GNN has been observed for years \cite{cai2020note}, a commonly accepted definition of OSM has just been proposed recently \cite{rusch2023survey}. The OSM in \cite{rusch2023survey} is defined as an exponential decay of the difference between node features. For example, asymptotically, one can show the after $t$ steps of iteration, the node feature $\mathbf H(t)$ generated from classic GCN model \cite{kipf2016semi} as: 
\begin{align}\label{GCN_solution}
    \mathbf H(t) = \mathrm{e}^{-t\mathbf L}\mathbf H(0) = \mathbf U\mathrm{e}^{-\boldsymbol{\Lambda}t}\mathbf U^\top \mathbf H(0),
\end{align}
where we denote $\mathbf L$ as the graph Laplacian, which admits an eigendecomposition $\mathbf L = \mathbf U \boldsymbol{\Lambda} \mathbf U^\top$, and $\boldsymbol{\Lambda} \in \mathbb R^{N\times N}$ is the diagonal matrix with entries of the eigenvalues of $\mathbf L$. One can observe that since $\mathbf L$ is semi-positive definite ($\boldsymbol{\Lambda}_{i,i} \geq 0$), with sufficient large of $t$, all features in $\mathbf H(t)$ tend to be identical. Consequently, all nodes will have the same prediction from GCN, which is known as the OSM problem. Obviously, one GNN with an OSM problem or consecutively smoothing the node features via its propagation is not preferred to fit the heterophilic graph in which many connected nodes have different labels. Therefore, to make GNN adapt to the heterophilic graph, an ideal GNN is to be able to induce both smoothing and sharpening effects on the node features. 

\cite{di2022graph} first explored this field by investigating the characteristics (monotonicity) of the filtering function incorporated in GNNs. Based on the conclusion from \cite{di2022graph}, the sharpening effect can be induced by leveraging the so-called high-pass filtering functions that monotonically increase the entries of the graph spectra or incorporating a transformation that generates negative eigenvalues. Accordingly, node features tend to be different based on Eq.~\eqref{GCN_solution}, and thus, \textbf{the OSM and graph heterophily problem are linked}. This observation has led to a range of follow-up studies \cite{han2022generalized,shi2023frameless,thorpe2022grand,shao2022generalized} that incorporate high-pass spectral filtering functions to mitigate both OSM and heterophily problems. Finally, it is worth noting that it can be verified that without considering the feature embedding process, if a GNN has a pure high-pass filtering function, then the node features will become more distinct from each other from the beginning of the propagation, then there will be no OSM problem \cite{di2022graph} in this GNN.

\paragraph{OSQ and the Trade-off Between OSM and OSQ}
The phenomenon of OSQ or bottleneck has just been identified recently \cite{alon2020bottleneck}, and unlike OSM, the OSQ problem has not obtained its commonly shared definition \cite{shi2023exposition}. Although OSQ has been measured via various indicators, all these methods are incorporated to measure the so-called sensitivity (or named as OSQ score mentioned in \cite{shi2023exposition}) between distanced nodes, that is: 
\begin{align}
  {\rm OSQ}(i,s) = \left\| \frac{\partial \mathbf{h}_i^{(\ell)}}{\partial \mathbf x_s} \right \|, 
\end{align}
where we denote $\ell$ as the number layers and $\|\cdot\|$ is spectral norm of a matrix. The key principle to check whether a GNN is capable of mitigating the OSQ problem is to check whether the upper bound of the OSQ score is increased. More specifically, from \cite{topping2021understanding}, the OSQ score is bounded by
\begin{align}\label{osq_problem}
        \left \|\frac{\partial \mathbf{h}_i^{(r+1)}}{\partial \mathbf x_s}\right \| \leq (\alpha\beta)^{r+1}({\mathbf{A}}^{r+1})_{is},
    \end{align}
where $\mathcal S_{r+1}(i):=\{j\in \mathcal V : d_\mathcal G(i,j) = r+1\}$  is the set of neighbours of node $i$ whose standard shortest path distance to node $i$ on graph $\mathcal G$, written as $d_\mathcal G(i,j)$, is $r+1 \in \mathbb N$. Note that ${\mathbf{A}}^{r+1}$ is ${\mathbf{A}}$ to the power of $r+1$. Therefore, one can see that one intuitive way of mitigating the OSQ problem is to \textbf{make the graph denser}, and this observation has motivated various graph rewiring approaches on OSQ based on different graph indicators, see Section \ref{related_works} for more details.

}

\section{Motivations and Model Formulation}

\subsection{Graph Diffusion as Particle System}
Let us go deeper to the scheme of GRAND presented in Eq.~\eqref{grad}, one can rewrite Eq.~\eqref{grad} into a component-wise form such that 
\begin{align}\label{component-wise_grand}
    \frac{\partial}{\partial t}\mathbf h_i = \sum_{j\in \mathcal N_i} a(\mathbf h_i, \mathbf h_j)(\mathbf h_j - \mathbf h_i), 
\end{align}
where we drop time $t$ from now on for the ease of notation. It suggests that the dynamic of the change of the feature of node $i$ is conducted by aggregating its neighbouring information, suggesting a homogenizing process on the connected nodes.  Furthermore, similar to the work in \cite{wang2022acmp}, one can interpret Eq.~\eqref{component-wise_grand} as an interactive particle system, in which all particles (nodes) attracted each other and eventually, after sufficient propagation, collapsed into one overlapped node with fixed feature, as long as we have the similarity score (i.e., $ a(\mathbf h_i, \mathbf h_j)$) positive throughout the propagating process. Although such feature processing might be friendly to homophily graphs, it is not hard to see that it may not be necessary for heterophily graphs, in which adjacent nodes are preferred to be pushed apart from each other, resulting in more distinctive node features. Accordingly, the recent work \cite{wang2022acmp} enhances Eq.~\eqref{component-wise_grand} by including negative similarities (i.e., negative edge weights) such that 
\begin{align}\label{acmp}
    \frac{\partial}{\partial t}\mathbf h_i = \sum_{j\in \mathcal N_i} (a(\mathbf h_i, \mathbf h_j) - \beta_{i,j}) (\mathbf h_j - \mathbf h_i) + \delta \mathbf h_i(1-\mathbf h^2_i),
\end{align}
in which $\beta_{i,j}$ is leveraged to adjust the sign of $(a(\mathbf h_i, \mathbf h_j) - \beta_{i,j})$ so that the particle system can adopt both attractive and repulsive forces. The additional term $\mathbf h_i(1-\mathbf h^2_i)$ is the double-well potential that is widely used in quantum particle physics \cite{jelic2012double} to prevent the so-called energy explosion, serving as a physical barrier (bound) of the node feature variation. Although remarkable improvement in learning accuracy has been observed from the model defined in Eq.~\eqref{acmp}, it is still unknown that \textbf{what type of force is needed for a given node pair}. In fact, in \cite{wang2022acmp}, $\beta_{i,j}$ was simplified as a single constant hyper-parameter throughout the training. This opens a rich venue for future research to explore the criteria for determining the suitable type of force, which is measured by the sign of similarity for the given pair of nodes.

\subsection{The Ideal Universe} \label{sec:critera_ideal_universe}
To identify the criteria of determining the type of forces for a given node pair, ideally one shall prefer the particle system to be evolved as 
\begin{tcolorbox}
\textit{All nodes with the same labels shall eventually collapse to one node with identical features. Nodes with different labels shall be apart from each other with distinctive features.} 
\end{tcolorbox}
This requirement directly suggests leveraging the node labeling information as a guide to determine the type of force for one specific node pair during the training process. Below we introduce the notion of (label) collapsing nodes.

\begin{defn}[Collapsing Nodes]
{Let $\mathcal{V} = \{v_1, ..., v_N \}$ be the nodes in the input graph and $\tilde v$ be a node with known label information $\tilde y$. $\tilde v$ is a collapsing node (CN) if for any $v_i \in \mathcal{V}$ with known labels $y_i$, $\tilde v$ is connected to $v_i$ with positive weights if $\tilde y = y_i$ and negative weights if $\tilde y \neq y_i$ regardless of original connectivity.}
  
\end{defn}

The definition of collapsing nodes (CNs) suggests that by leveraging the known node label information, CNs serve as a ``gravitational''
source via the feature propagation, since all nodes with available label information are connected to CNs, and will be attracted/repulsed by CNs if they have the same/different labels. We further note that CNs can be either sourced from the node from the original graph or added as additional nodes to the graph. For the convenience of this study, in the sequel, we only consider the second type of CNs (additional nodes to the graph). Accordingly, this augments the original graph by adding a connection matrix $\mathbf C \in \mathbb R^{N \times K}$ (for $K$ collapsing nodes),  
\begin{equation*}
    C_{i,k} = \begin{cases}
        +1, &\text{ if } y_i\text{ known and } y_i = \tilde y_k \\
        -1, &\text{ if } y_i \text{ known and } y_i \neq \tilde y_k \\
        0, &\text{ otherwise }
    \end{cases}
\end{equation*}
where we use $\tilde y_k$ to denote the label of CN $k$ for $k = 1,..., K$. For convenience, we set $K$ equal to $C$, the number of unique labels in a given graph.
The corresponding adjacency matrix after including CNs is thus given by   
\begin{align}\label{form_of_Ac}
    \mathbf A_c = 
    \begin{bmatrix}
        \mathbf A  & \mathbf C \\
        \,\, \,\mathbf C^\top & \mathbf 0
    \end{bmatrix}.
\end{align}
Further if the adjacency includes self-loops, then $\mathbf A_c = 
    \begin{bmatrix}
        \mathbf A  \,\,\, &\mathbf C \\
       \mathbf C^\top  &\mathbf 1_{K\times K} 
    \end{bmatrix}$,
where $\mathbf 1_{K\times K}$ is a diagonal matrix (block) of size $K\times K$ with entries of all ones. Considering that $\mathbf A_c$ contributes valuable node labeling information to the propagation of node features, a \textbf{model-agnostic framework} can be established. 
We name such framework as \textbf{UYGNNs}, short for \textbf{U}niversal Label based (\textbf{Y}) \textbf{G}raph \textbf{N}eural \textbf{N}etworks. Accordingly, one can define \textbf{UYGCN} that propagates nodes features with the following dynamic
\begin{align}
    \frac{\partial}{\partial t}\mathbf h_i = \sum_{j\in \mathcal N_i} (a^{\mathrm{GCN}}_c(\mathbf h_i, \mathbf h_j))(\mathbf h_j - \mathbf h_i), \label{eq:uygcn}
\end{align}
where $a^{\mathrm{GCN}}_c(\mathbf h_i, \mathbf h_j)$ denotes the edge similarities contained in $\mathbf A_c$. Note that here the index $i$ runs through all the $N+K$ nodes 
and $\mathcal N_i$ includes the neighbours of node $i$ from the extended edge set $\widehat{\mathcal E} = \mathcal E_0 \bigcup \mathcal E_1$, where $\mathcal E_1$ denotes the set of edges of connecting nodes with label information available to CNs. 
Similarly, attention mechanism \cite{velivckovic2018graph} can be deployed into the above massage passing scheme. In this case, one can further define the dynamic of \textbf{UYGAT} as
\begin{align}\label{uygat}
    \frac{\partial}{\partial t}\mathbf h_i = \sum_{j\in \mathcal N_i} &\big(a^{\mathrm{GCN}}_c(\mathbf h_i, \mathbf h_j)\cdot a^{\mathrm{GAT}}_c(\mathbf h_i, \mathbf h_j)\big) \cdot\big(\mathbf h_j - \mathbf h_i\big),
\end{align}
where $a^{\mathrm{GAT}}_c(\mathbf h_i, \mathbf h_j) >0$ is the attention coefficient for the connected nodes pair $(i,j)$. The multiplication between $a^{\mathrm{GCN}}_c$ and $a^{\mathrm{GAT}}_c$ is leveraged to preserve the sign of the edge weights so that the type of forces can be maintained. Furthermore, to prevent the potential energy explosion in UYGNN models, similar to ACMP \cite{wang2022acmp}, double-well potential as well as other type of physical potential terms (e.g., harmonic oscillator potential \cite{marsiglio2009harmonic}) can be leveraged to further restrict the motion of nodes.

\subsection{An even More Complexed Universe}
Apart from the structure of UYGCN and UYGAT, which only assigns repulsive force between NC and the nodes with different labels, node labeling information can also serve as a guide to assign repulsive forces on original graph connectivities. Specifically, one can achieve this by leveraging the label-based adjacency matrix $\mathbf A_y$, with its entries defined as
\begin{equation*}
    (\mathbf {A}_y)_{i,j} = \begin{cases}
        +1, &\text{ if } y_i = y_j \\
        -1, &\text{ if }  y_i \neq y_j \\
        0, &\text{ otherwise }
    \end{cases}
\end{equation*}
for any edge pair $(i,j) \in \mathcal E$. {We highlight that the term ``otherwise'' in the above expression stands for the case that either or both nodes other than CNs are without labeling information.} The corresponding dynamic is 
\begin{align}
    \frac{\partial}{\partial t}\mathbf h_i 
    =&\sum_{j\in \mathcal N^0_i} (a_y(\mathbf h_i,\mathbf h_j) \cdot a(\mathbf h_i, \mathbf h_j))\cdot(\mathbf h_j - \mathbf h_i) 
    + \sum_{j\in \mathcal N^1_i} a_c(\mathbf h_i, \mathbf h_j)\cdot(\mathbf h_j - \mathbf h_i),
\end{align} 
where we denote $\mathcal N^0_i$ and $\mathcal N^1_i$ as the neighbors of node $i$ from $\mathcal E_0$ and $\mathcal E_1$, respectively. It is not hard to see that for any nodes with label information above, such universe design will ensure that nodes will \textbf{only} be attracted by those nodes with the same label and pushed away from nodes with different labels.

\subsection{MLP in\&out Paradigm}
To implement the propagation according to the dynamic UYGNNs (i.e., Eq.~\eqref{eq:uygcn}), one can conduct an MLP in-and-out paradigm \cite{chamberlain2021grand} as follows. 
\begin{align}
    &\mathbf H^{(\ell)} = \mathrm{MLP}^{(\ell)}_{\mathrm {in}} (\mathbf H^{(\ell)}), \label{eq: mlp_in} \\
    & \mathbf H' = \sigma_\ell (\mathbf A_c \mathbf H^{(\ell)}), \label{eq: propagation} \\
    & \mathbf H^{(\ell +1) } = \mathrm{MLP}^{(\ell)}_{\mathrm {out}} (\mathbf H'), \label{eq: mlp_in_out}  
\end{align}
where we have $\mathbf H^{(0)} = \mathbf X$. Practically, $\mathrm{MLP}_{\mathrm {in}} $ and $\mathrm{MLP}_{\mathrm {out}} $ are implemented with two learnable channel mixing matrices $\mathbf W^{(\ell)}_{\mathrm {in}} \in \mathbb R^{d_\ell \times d'_\ell}$ and  $\mathbf W^{(\ell)}_{\mathrm {out}} \in \mathbb R^{d'_\ell \times d_{\ell+1}}$, respectively, and $\sigma$ is the activation function. One can further check that the computational complexity of UYGCN is with $\mathcal O(|\widehat{\mathcal E}|d'_\ell)$, which is similar to the classic GNNs \cite{kipf2016semi}.

\section{Theoretical Analysis of YGNNs}
In this section, we theoretically verify that UYGCN can handle the aforementioned GNN issues, such as OSM, OSQ, and heterophily adaption.

\subsection{Avoiding OSM}\label{sec:avoiding_osm}
We start our analysis by showing that UYGCN can avoid the OSM problem. Specifically, the dynamics of Eq.~\eqref{eq:uygcn} can be shown to avoid OSM in the limit, i.e., not converging to a constant state. This can be read directly from the result that the constant state is not a stationary point of the dynamics. To see this, we consider normalized adjacency $\widehat{\mathbf A}_c = \mathbf D_c^{-1} (\mathbf A_c + \mathbf I)$ (for the symmetrically normalized adjacency, the idea is the same), where $(\mathbf D_c)_{ii} = \sum_j |(a_c)_{i,j}|$. Without loss of generality, consider a single feature $\mathbf h = [\mathbf x; \mathbf f]$ where $\mathbf x \in \mathbb R^{N}, \mathbf f \in \mathbb R^{K}$ represents the features of original nodes and added collapsing nodes.  Thus Eq.~\eqref{eq:uygcn} can be viewed as a system of $(\mathbf x, \mathbf f)$. In order for OSM to occur, there must exist a constant vector $\mathbf c = c \boldsymbol{1}_N$  such that $\mathbf x= \mathbf c$ as $t \rightarrow \infty$.

The next proposition shows that the constant state is not a limiting state of the dynamics we consider as long as there exists at least one training sample per class. This suggests OSM can be avoided.

\begin{prop}
\label{prop:oversmooth}
Suppose $K$ is equal to the number of classes and there exists at least one training sample per class.  Consider the Euler discretized dynamics of Eq.~\eqref{eq:uygcn} as the fixed point iteration $\mathbf h^{(\ell+1)} = \mathbf A_c \mathbf h^{(\ell)}$ where $\mathbf h\in\mathbb{R}^{N+K}$.  Then the limit $\mathbf h = (\mathbf c, \bar {\mathbf f})$ of the iteration, the stationary point, is not in the form of $\mathbf c = c \boldsymbol{1}_N\in\mathbb{R}^N$ and $\bar{\mathbf f} \in \mathbb R^{K}$.
\end{prop}
\begin{proof}
The proof is by contradiction.  Suppose that the fixed point of its discrete iteration of Eq.~\eqref{eq:uygcn} is
$(\mathbf c, \bar{\mathbf f})$ with $\mathbf c = c\mathbf I_N$, then we have
\begin{equation*}
    \begin{cases}
         \mathbf c = \widehat{\mathbf A} \mathbf c + \widehat{\mathbf C}  \bar{\mathbf f},\\
         \bar{\mathbf f} = \widehat{\mathbf C}^\top \mathbf c + \mathbf D_f^{-1} \bar{\mathbf f},
    \end{cases}
\end{equation*}
where $\widehat{\mathbf A} = (\widehat{\mathbf A}_c)_{1:N, 1:N}$ and $\widehat{\mathbf C} = (\widehat{\mathbf A}_c)_{N:N+K, 1:N}$ and $\mathbf D_f \in \mathbb R^{K \times K}$ be the diagonal degree matrix of the added nodes. 
For each added node $k$, suppose we can find one sample $i \in \mathcal{V}_{\rm tr}$ such that $y_i = y_k$. Then by the first equation, we have 
\begin{align*}
    K c &= \bar f_1 - \sum_{k \neq 1} \bar f_k, \\
    Kc &= \bar f_2 - \sum_{k \neq 2} \bar f_k, \\
    \cdots \\
    Kc &= \bar f_K - \sum_{k \neq K} \bar f_k.
\end{align*}
This suggests $\bar f_1 = \bar f_2 = \cdots \bar f_K = \bar f =  -\frac{K}{K-2}c$. From the second equation, without loss of generality,  we consider a single added node with degree $d= 1 + n^+ + n^-$ where $n^+, n^-$ denote the number of training nodes with a same or different label as the considered node. Hence, this suggests
\begin{equation*}
    \bar f = \frac{1}{d} (n^+ - n^{-} )c + \frac{1}{d} \bar f,
\end{equation*}
which leads to $\bar f = \frac{n^+ - n^-}{n^+ + n^{-}} c = - \frac{K}{K-2} c$. This gives rise to a contradiction that  $n^+ = - \frac{2}{2K-1} n^- < 0$, because $K \geq 1$. The proof is complete.
\end{proof}
In addition, One can verify that on one specific time $t$, the dynamic in Eq.~\eqref{eq:uygcn} is the minimizer of the following energy

\begin{align*}
    E^{-}(\mathbf H) &= {\rm tr}(\mathbf H^\top (\mathbf I - \mathbf A_c ) \mathbf H) = \sum_{i,j} (\mathbf A_c)_{i,j} \| \mathbf h_i - \mathbf h_j \|^2 \\
    &= \sum_{(i,j) \in \mathcal{E}_0} \| \mathbf h_i - \mathbf h_j \|^2 + \sum_{(i,k) \in \mathcal E_1, y_i = y_k} \!\!\!\!\!\!\!\! \|\mathbf h_i - \mathbf h_k \|^2 
    -  \sum_{(i,k) \in \mathcal E_1, y_i \neq y_k} \!\!\!\!\!\!\!\! \| \mathbf h_i - \mathbf h_k \|^2.
\end{align*}
This provides an intuitive explanation for the behaviour of our proposed dynamics, i.e., pushing samples from the same class to its corresponding collapsing node, while maximizing their separation from other collapsing nodes. We further highlight that if we combine the energy $E^{-}(\mathbf H)$ with the double-well potential term, then the energy becomes the popular Ginzburg-Landau energy \cite{luo2017convergence}. However, $E^{-}(\mathbf H)$ can be negative depending on the graph spectra, we provide a more detailed discussion on this phenomenon regarding GNN heterophily adaption in Section \ref{sec:negative_lap_eigenvalues}.

\paragraph{Cluster Flocking}\label{sec:flocking}
At the beginning of Section \ref{sec:critera_ideal_universe}, we proposed the ideal evolution of the node features according to their labeling information, and we have verified that, based on our design, nodes propagated under UYGNNs move toward the consequences that we prefer to observe. Yet it is still unknown whether node features' asymptotic states align with the criteria we proposed before. In Appendix \ref{append: cluster_flocking}, 
we demonstrate that the UYGAT  in Eq.~\eqref{uygat} with double-well potential can asymptotically achieve the so-called \textit{cluster flocking} \cite{fang2019emergent} in which nodes with same labels will be clustered due to the attractive forces whereas nodes with different labels will be parted away because of the repulsive forces. However, it is unlikely for UYGCN to achieve cluster flocking. Similarly, determining the graph structure that guarantees node clustering flocking is a complex but exciting field to explore, and we leave it to future works.

\paragraph{Deal with Energy Explosion}
Similar to ACMP \cite{wang2022acmp}, the double-well potential helps to regulate the magnitude of Dirichlet energy, as shown in the following proposition. 

\begin{prop}
\label{prop:energy_explode}
Denote $\mathcal{E}_{\rm dir}(\mathbf h) = \mathbf h^\top \widehat {\mathbf L} \mathbf h$ the Dirichlet energy of the signal on graph $\mathcal{G}$. Then there exist a constant depending only on $N$ and the largest eigenvalue $\lambda_{\max}$ of $\mathbf L$, such that $\| \mathbf h\|_2^2 \leq C$ and $\mathcal{E}_{\rm dir}(\mathbf h) \leq \lambda_{\max} \| \mathbf h\|_2^2 \leq 2 \| \mathbf h\|_2^2 $, for some constant that depends on the node size.  
\end{prop}
\begin{proof}
By using the similar way in proving Proposition 1 of \cite{wang2022acmp}, we can prove that $\| \mathbf h\|^2_2$ is bounded by a constant $C$ depending on the graph size and the largest eigenvalue of the Laplacian $\mathcal{L}$. Therefore $\mathcal{E}_{\rm dir}(\mathbf h) \leq \lambda_{\max} \| \mathbf h \|_2^2 \leq 2C$. 
\end{proof}

\begin{tcolorbox}
{\textit{\textbf{Takeaway message of this section:} UYGNN avoids OSM problem by introducing repulsive forces between nodes, and similar to the method leveraged in particle physics, double-well potential can be used to prevent the model from energy explosion.} }
\end{tcolorbox}

\subsection{Role of Laplacian Negative Eigenvalues and Heterophily Adaption}\label{sec:negative_lap_eigenvalues}

In this section, we delve deeper into the impact of the negative weighted edges on the graph spectra. Specifically, one can define the graph Laplacian $\mathbf L_c = \mathbf E \mathbf W_c \mathbf E^\top$ \cite{chen2016characterizing}, where $\mathbf E \in \mathbb R^{(N+K)\times |\widehat{\mathcal E}|}$
is the incidence matrix and $\mathbf W_c \in \mathbb R^{|\widehat{\mathcal E}|\times |\widehat{\mathcal E}|} $ is a diagonal matrix with entries of edge weights . Based on the form of the dynamics of UYGCN in Eq.~\eqref{eq:uygcn}, without considering the double-well potential term, one can explicitly write out the solution of the differential equation as\footnote{We further assume $\mathbf L_c$ has distinct eigenvalues.} 

\begin{align}\label{UYGCN_solution}
    \mathbf H(t) = \mathrm{e}^{-t\mathbf L_c}\mathbf H(0) = \mathbf U \mathrm e^{-\boldsymbol{\Lambda}_c t} \mathbf U^\top \mathbf H(0),
\end{align}
where $\boldsymbol{\Lambda}_c$ is the diagonal matrix with entries of eigenvalues of the Laplacian. By assuming a discrete integer time $t$ and replacing it with the number of layer $\ell$, one can have the solution of the UYGCN, which serves as the discretized version of the dynamics in Eq.~\eqref{eq:uygcn}. Now, as we have included negative edge weights in $\mathbf A_c$, $\mathbf L_c$ could have negative eigenvalues \cite{chen2016characterizing}. We highlight that this is exactly what UYGNN needs to fit heterophily graphs. Since from Eq.~\eqref{UYGCN_solution}, one can treat $\mathrm{e}^{-\boldsymbol{\Lambda}_c}$ as a spectral filter, and if the entries of $\boldsymbol{\Lambda}_c \geq 0$, then $\mathrm{e}^{-\boldsymbol{\Lambda}_c}$ serves as a low-pass filtering function (monotonic decrease), which tends to homogenize the components of node features, whereas in the case when all $(\boldsymbol{\Lambda}_c)_{ii} < 0$ it becomes a high-pass filtering function (monotonic increase), which imposes sharpening effect on nodes features. Let $\mathcal G_c$ be the graph with CNs additionally added, $\mathcal V_c$ be the set of nodes of $\mathcal G_c$, and $\mathcal T(\mathcal V_c)$ be the training set. The following theorem provides insights into its Laplacian spectrum. 
\begin{thm}\label{eigen_value}
    Assuming every additionally added CN has at most one node in the training set $\mathcal T(\mathcal V_c)$ that shares the same label, then the number of negative eigenvalues of $\mathbf L_c \in \mathbb R^{(N+K)\times (N+K)}$ is between 0 and $\frac{K(K-1)}{2} + K(|\mathcal T(\mathcal V_c)|-1)$. 
\end{thm}
\begin{proof}
Denote $\mathcal{E}^{-}$ all the negative edges in the augmented graph $\mathcal{G}_c$. Assume that every additional CN has at least one node sharing identical label, then the number of newly added negatively weighted edges connecting the CN would be less than $|\mathcal T(\mathcal V_c)-1|$. Since CNs are themselves connected with negative edges, we have $K(K-1)/2$ number of such edges. Therefore the total number of negatively weighted edges $|\mathcal{E}^{-}| \leq K(K-1)/2 + K|\mathcal T(\mathcal V_c)-1|$. Theorem 7 in \cite{song2019extension} shows that $\mathbf L_c$ has at most $|\mathcal{E}^{-}|$ negative eigenvalues, which is less than $K(K-1)/2 + K|\mathcal T(\mathcal V_c)-1|$ negative eigenvalues. This completes the proof.  
\end{proof}

One important observation from Theorem \ref{eigen_value} is although edges with negative weights are included, $\mathbf L_c$ can still be SPD, i.e., $(\boldsymbol{\Lambda})_{ii} \geq 0$ for all $i$. In this case, based on Eq.~\eqref{UYGCN_solution}, $\mathrm{e}^{-\boldsymbol{\Lambda}}$ will remain as a low-pass filtering
This suggests that UYGCN inherits the feature smoothing property of the classic GCNs \cite{kipf2016semi,velivckovic2018graph}. 
On the other hand, under the scenario such that $\frac{K(K-1)}{2} + K(|\mathcal T(\mathcal V_c)|-1) > N+K$, meaning all eigenvalues of 
$\mathbf L_c $ are negative,
then  $\mathrm{e}^{-\boldsymbol{\Lambda}}$ will behave like a high-pass filtering function. This indicates that UYGCN can induce both smoothing and sharpening dynamics and is thus able to fit both homophilic and heterophilic graphs. Furthermore, the necessary and sufficient condition for the definiteness of $\mathbf L_c$ is closely related to the topology of the so-called electrical network \cite{chen2016characterizing} of the input graph and is out of the scope of this paper, so we omit it here.

\begin{tcolorbox}
{\textit{\textbf{Takeaway message of this section:} 
Avoiding the OSM problem naturally leads UYGNN to handle heterophilic graphs. With those negative eigenvalues in the graph spectra induced by the CNs, UYGNN can induce both smoothing (from the low-pass filtering component) and sharpening (from the high-pass filtering component) effects on the node features, thus adapting both homophilic and heterophilic graphs.
} }
\end{tcolorbox}

\subsection{Deal with OSQ}
As CNs bring additional connectivities to the original graph, it is natural to explore the impact of these newly added edges on the so-called OSQ problem of GNNs. The OSQ issue of GNNs can be viewed as a type of information compression problem such that a large amount of feature information is compressed into a narrow path due to the graph topology, causing GNNs to fail to capture the long-term relationship between nodes \cite{topping2021understanding,black2023understanding}. In this section, we measure the OSQ problem via the sensitivity score leveraged in \cite{topping2021understanding,black2023understanding,shi2023exposition}, and show that with the help of CNs, UYGNNs could increase the OSQ score upper bound thus mitigate the OSQ problem. We further highlight that despite CNs bringing both attractive and repulsive forces by assigning different signs on edges, from the information transaction point of view, both types of these additional edges make the information ``easier'' to communicate. Accordingly, it is natural to measure the OSQ problem via $|\mathbf A_c|$ in which additional edges with negative weights are replaced by their absolute values. It is worth noting that, for the sake of convenience in analysis, we only consider one channel mixing matrix (i.e., $\mathbf W_{\mathrm{in}}$), and our conclusion can be easily extended to the case for both $\mathbf W_{\mathrm{in}}$ and $\mathbf W_{\mathrm{out}}$.

\begin{prop}\label{UYGNN_OSQ}
   Consider the UYGNN paradigm defined in Eq.~\eqref{eq: mlp_in} and Eq.~\eqref{eq: propagation}. 
    Assuming $\mathcal G$ is not a bipartite graph,
    let $i,s \in \mathcal V$, if
    $|\sigma'_\ell| \leq \alpha$, $\| \mathbf W_{\mathrm{in}}^{(\ell)}\| \leq \beta_{\mathrm{in}}$ for $0 \leq \ell \leq r$, then 
    \begin{align}\label{OSQ_score}
        \left \|\frac{\partial \mathbf{h}_i^{(r)}}{\partial \mathbf x_s}\right \| \!\leq \!(2\alpha\beta_{\mathrm{in}})^{r} \left(\sum_{\ell=0}^r |\mathbf A_c|^\ell\right)_{i,s}.
    \end{align}     
\end{prop}
\begin{proof}[Proof of Proposition \ref{UYGNN_OSQ}]
   The proof follows directly from the process used in \cite{black2023understanding}. If there is no CN either from the nodes of the original graph or additionally added, the above equation is reduced to the conclusion provided in \cite{black2023understanding}. To show the inclusion of CNs can actually increase the OSQ upper bound, if CNs are selected from the original graph, i.e., the block $\mathbf C$ in $\mathbf A_c$ disappears, and $|\mathbf A_c|$ becomes denser compared to $\mathbf A$ due to newly added edges. Therefore it naturally mitigates the OSQ problem according to Eq.~\eqref{OSQ_score}.
   More specifically, one can verify that if these CNs (selected from the original graph) are not connected to any other nodes, then the maximum number of edges added is  $K|\mathcal T(\mathcal V)-1|$ where we let $|\mathcal T(\mathcal V)|$ be the number of nodes in the training set in which all nodes labeling information are available. In the worst case, if all of these types of CNs are connected and connected to all other nodes in the training set, then based on Eq.~\eqref{OSQ_score}, same upper bound can be ensured. Similar reasoning can be conducted when CNs are additionally added to the graph, and the number of edges added is fixed as $K|\mathcal T(\mathcal V)| + \frac{K(K-1)}{2}$. 
   \end{proof}
   
   Lastly, it is worth noting that in our proof and practical implementations, we only considered two cases that is (1) All CNs are sourced from the graph, and (2) All CNs are new to the graph. For the cases where the source of CNs is mixed, we omit it here for simplicity. 

\begin{tcolorbox}
{\textit{\textbf{Takeaway message of this section:} 
UYGNN naturally handles OSQ problems with its graph rewiring scheme, which makes information transactions ``easier'' than before. We further highlight that the OSQ problem purely depends on the sensitivity between node features rather than the sign of the edge weights based on Eq.~\eqref{OSQ_score}.} }
\end{tcolorbox}

\subsection{Impact Model's Rewiring Scheme on Curvatures and Trade-off Between OSM and OSQ}
Recent studies have verified that the trade-off between OSM and OSQ issues is closely related to the so-called discretized curvatures defined on the edge of the graph \cite{fesser2023mitigating,giraldo2022understanding}. Since additional edges induced by the CNs enrich the higher order structure of the graph i.e., triangles and 4-cycles, thus the direct consequence of this is that the curvature on the edge becomes larger. For example, consider the Augmented Forman curvature \cite{fesser2023mitigating} defined as 
$$FC_3(i,j):= 4- d_i - d_j + 3|\sharp_\Delta(i,j)|,$$ 
where $|\sharp_\Delta(i,j)|$ denote as the number of triangles that contain edge $(i,j)$. If we add one CN, say node $k$, into the graph, both node $i$ and $j$ will be rewired to connect to the CN, thus naturally we have one additional triangle in the graph. Therefore, the corresponding curvature $FC'_3(i,j)$ on the edge $(i,j)$ becomes $FC'_3(i,j) = 4-(d_i+1)-(d_j+1) + 3(|\sharp_\Delta(i,j)|+1)$ which is larger than the original $FC_3(i,j)$. Similar conclusions can be drawn when one considers higher-order structures such as 
4-cycles
Since negative curvatures are responsible for the OSQ problems \cite{topping2021understanding,fesser2023mitigating,giraldo2022understanding}, and the inclusion of CNs increases curvatures by its induced additional connectivities, therefore naturally mitigates the OSQ problem. 

Regarding the OSM issue, intuitively, one can consider the following four situations between original graph nodes $i,j$, and CNs: (1) The additional CN (denoted as node $k$ for short) has the same label with both nodes $i,j$; (2) The label of node $k$ only aligns with one node either from node $i$ or $j$; (3) Nodes $i$ and $j$ shares the same label whereas node $k$ is with a different label; (4) All three nodes have different labels. Based on the analogy between GNN propagation and particle systems, one can find that only in situation (1), all three nodes will eventually share the same feature, which is preferred due to their labeling information. In all other situations, nodes with different labels will be pushed away by repulsive forces induced by CNs, which further dilutes the non-desirable attractive forces from their original neighbor that has different labels, suggesting a better fitness compared to the system with no repulsive force, and an analogy of continuous shape deformation of surfaces \cite{jin2008discrete}. However, to appropriately quantify the impact of CNs on the graph curvature as well as the trade-off relation between OSM and OSQ, one shall be required to re-define the curvature over the graph with negatively weighted edges (i.e., sign graphs), we leave this as future work.

\begin{tcolorbox}
{\textit{\textbf{Takeaway message of this section:} 
We show the impact of the graph rewiring scheme in UYGNN on the graph (Forman) curvature, which often serves as an indicator for OSM and OSQ problems. The CNs in UYGNN induce more higher-order graph structures (i.e., triangles, 4-cycles) for mitigating the OSQ problem, and those negatively weighted edges induced from CNs make UYGNN avoid the OSM problem.} }
\end{tcolorbox}

\section{Experiment}

In this section, we conducted numerical experiments to test our proposed models. Specifically, in Section \ref{normal_node_classification}, we test UYGNNs via three homophilic (\texttt{Cora, Citeseer, Pubmed}) and three heterophilic graphs (\texttt{Cornell, Texas, Wisconsin}) and one large-scale benchmark (\texttt{Ogbn-Arxiv}) \cite{hu2020open}. In addition, we also test our models' performance via the increase of layers to verify its robustness on the OSQ problem as aforementioned in Section \ref{sec:avoiding_osm}. Further in Section \ref{sensitivity}, we analyze the sensitivity of our model to the number of CNs. In addition, in Section \ref{experiment_lrgb}, we show the performance of UYGNNs over long-range graph benchmarks (LRGB) provided in \cite{dwivedi2022long} to verify its advantage on the OSQ problem. We provide more experimental details in Appendix \ref{append:experimental_details}. {All of our experiments are conducted in Python 3.11 with Pytorch Geometric on the Gadi supercomputing system from National Computational Infrastructure (NCI) Australia, which contains 2 nodes of the NVIDIA DGX A100 system, with 8 A100 GPUs per node.}

\subsection{Node Classification on Homophily/Heterophily Graphs}\label{normal_node_classification}

\paragraph{Setup}
For the settings in UYGNNs, we initially choose CNs by adding $K$
nodes with their features generated by one layer $\mathrm{MLP}$. For UYGCN, the form of $\mathbf A_c$ is used throughout the whole training process, while in UYGAT, the (single-head)
attention mechanism is conducted to re-weight $\mathbf A_c$ with attention coefficients \cite{velivckovic2018graph}. For both UYGCN and UYGAT, we normalized $\mathbf A_c$ (with self loop) with $\mathbf D_c^{-\frac12} \mathbf A_c \mathbf D_c^{-\frac12}$,
where $(\mathbf D_c)_{ii} = \sum_j |(a_c)_{i,j}|$. Both UYGCN and UYGAT are implemented with two layers. The grid search approach is conducted for fine-tuning model hyperparameters. 
Both methods are trained with the ADAM optimizer. The maximum number of epochs is 200 for citation networks and heterophilic graphs, whereas 500 for \texttt{Ogbn-Arxiv}. The descriptive statistics of the included datasets with their homophily indices are listed in Table.~\ref{data_statistics}. All the datasets follow
the standard public split and processing rules. The average test accuracy and its standard deviation come from 10 runs. 
\begin{table}[h]
\centering
\caption{Statistics of the datasets, $\mathcal H(\mathcal G)$ represent the level of homophily. \\}
\label{data_statistics}
\setlength{\tabcolsep}{8pt}
\renewcommand{\arraystretch}{1}
    \begin{tabular}{ccrrrr}
    \hline
         Datasets & Class & \multicolumn{1}{c}{Feature} &  Node\ \  & Edge\ \ & \multicolumn{1}{c}{$\mathcal H(\mathcal G)$}\\
         \hline
        Cora & 7 & 1433 & 2708 & 5278  & 0.825\\
        Citeseer & 6 & 3703 & 3327 & 4552   & 0.717\\
        PubMed & 3 & 500 & 19717 & 44324  & 0.792\\
        Arxiv & 23 & 128 & 169343 &  1166243   & 0.681\\
        
        Wisconsin &5 &251 &499 &1703  &0.150\\
        Texas &5 &1703 &183 &279 &0.097\\
        Cornell &5 &1703 &183 &277  &0.386\\
        \hline
    \end{tabular}   
\end{table}

\begin{table*}[t]
\centering
\caption{Performance of UYGCN and UYGAT on homophilic, heterophilic, and large scale graph dataset (Arxiv). The best performance is highlighted in \textbf{bold}, and the second best performance is underlined. }
\label{table:node_classification}
\setlength{\tabcolsep}{9pt}
\renewcommand{\arraystretch}{1.5}
\vskip 0.15in
\scalebox{0.75}{
\begin{tabular}{clllllll}
\hline
Methods    & \multicolumn{1}{c}{Cora} & \multicolumn{1}{c}{Citeseer} & \multicolumn{1}{c}{Pubmed} & Cornell & Texas & Wisconsin & Arxiv \\ \hline
MLP       &55.1\(\pm\)1.3                 &59.1\(\pm\)0.5                             &71.4\(\pm\)0.4                            &\underline{91.3\(\pm\)0.7}         &\textbf{92.3\(\pm\)0.7}       &\underline{91.8\(\pm\)3.1}          &55.0\(\pm\)0.3   \\ 
GCN       &81.5\(\pm\)0.5                  &70.9\(\pm\)0.5                              &79.0\(\pm\)0.3                            &66.5\(\pm\)13.8         &75.7\(\pm\)1.0       &66.7\(\pm\)1.4           &72.7\(\pm\)0.3        \\
GAT      &83.0\(\pm\)0.7 &72.0\(\pm\)0.7                              &78.5\(\pm\)0.3                            &76.0\(\pm\)1.0         &78.8\(\pm\)0.9       &71.0\(\pm\)4.6           &72.0\(\pm\)0.5      \\
GIN  &78.6\(\pm\)1.2                         &71.4\(\pm\)1.1                              &76.9\(\pm\)0.6                            &78.0\(\pm\)1.9         &74.6\(\pm\)0.8       &72.9\(\pm\)2.5           &64.5\(\pm\)2.5       \\

APPNP  &83.5\(\pm\)0.7                         &\underline{75.9\(\pm\)0.6}                              &79.0\(\pm\)0.3                           &\textbf{91.8\(\pm\)0.6}         &83.9\(\pm\)0.7       &92.1\(\pm\)0.8           &70.3\(\pm\)2.5\\

H2GCN  &83.4\(\pm\)0.5                         &73.1\(\pm\)0.4                              &79.2\(\pm\)0.3                            &85.1\(\pm\)6.1         &85.1\(\pm\)5.2       &87.9\(\pm\)4.2           &72.8\(\pm\)2.4        \\

GPRGNN  &83.8\(\pm\)0.9                         &75.9\(\pm\)1.2                              &79.8\(\pm\)0.8                            &85.0\(\pm\)5.2         &75.9\(\pm\)9.2       &90.4\(\pm\)3.0           &70.4\(\pm\)1.5        \\

LEGCN  &81.9\(\pm\)2.1                         &73.2\(\pm\)1.4                              &77.4\(\pm\)0.5                            &81.2\(\pm\)3.6         &81.8\(\pm\)2.9       &71.5\(\pm\)1.3           &\underline{73.3\(\pm\)0.3}   \\

Replusion  &82.3\(\pm\)0.8                         &71.9\(\pm\)0.5                              &79.3\(\pm\)1.2                            &86.3\(\pm\)2.8         &83.9\(\pm\)1.5       &86.6\(\pm\)4.2           &71.9\(\pm\)0.4   \\

GRAND     &82.9\(\pm\)1.4                          &70.8\(\pm\)1.1                              &79.2\(\pm\)1.5                            &72.2\(\pm\)3.1         &80.2\(\pm\)1.5       &86.4\(\pm\)2.7           &71.2\(\pm\)0.2       \\
SJLR      &81.3\(\pm\)0.5                     &70.6\(\pm\)0.4                              &78.0\(\pm\)0.3                            &71.9\(\pm\)1.9         &80.1\(\pm\)0.9       &66.9\(\pm\)2.1           &72.0\(\pm\)0.4       \\
ACMP  &84.6\(\pm\)0.5                     &75.0\(\pm\)1.0                              &78.0\(\pm\)0.3                            &84.3\(\pm\)4.8         &85.4\(\pm\)4.2       &87.8\(\pm\)3.3           &68.9\(\pm\)0.3  \\

\hline
UYGCN      &\textbf{84.8\(\pm\)0.3 }                     &75.2\(\pm\)0.4                          &  \textbf{79.9\(\pm\)0.5}                          &85.7\(\pm\)1.6      &88.9\(\pm\)1.5        &\textbf{93.6\(\pm\)2.7 }         &\textbf{74.4\(\pm\)0.9 }          \\

UYGAT      &84.0\(\pm\)0.4                         &\textbf{76.1\(\pm\)0.8}                              &\underline{79.6\(\pm\)1.5}                            &87.4\(\pm\)1.3         &\underline{89.8\(\pm\)2.5}       &89.9\(\pm\)1.8           &72.3\(\pm\)0.3\\
 \hline
\end{tabular}}
\end{table*}

\paragraph{Baselines}
We compare UYGNNs with various baseline models. Except for those classic baseline models such as GCN \cite{kipf2016semi} and GAT \cite{velivckovic2018graph}, we also include some recent works such as LEGCN \cite{yu2022label}, Repulsive GNN \cite{gao2022repulsion}, which respectively leverage labeling information and repulsive forces to enhance GNNs performances. Furthermore, we also include SJLR \cite{topping2021understanding}, which serves as the initial work of measuring the OSQ problem and resolving it using the graph rewiring paradigm. Finally, we include ACMP, which is the first analogizing the dynamic of GNNs as an evolution of particle systems. All baseline performances are retrieved from public results, and if the results are not available, we implement baseline models to produce learning accuracy with our best effort. 

\paragraph{Results}
We report the accuracy score percentage with the top 2 highlighted in Table \ref{table:node_classification}. One can find that both UYGCN and UYGAT achieved remarkable learning accuracy compared to the baseline models, especially when compared to their original counterparts, i.e., UYGCN compared to GCN and UYGAT compared to GAT. This observation suggests that our approach of leveraging node labeling information and repulsive force is a model agnostic method to enhance classic GNNs. Furthermore, our models show a good fit on both homophilic and heterophilic datasets thanks to the repulsive forces and the definiteness of graph Laplacian in Theorem \ref{eigen_value}. Lastly, one can observe that UYGAT is, in general with better performances compared to UYGCN in terms of heterophilic graphs; this supports our claim on its\textit{ bi-cluster flocking} property that we mentioned in Section \ref{sec:flocking}. 

\paragraph{Additional experiment on OSM}
In Section \ref{sec:avoiding_osm}, we demonstrated that our UYGNNs can circumvent the OSM issue. In this section, we verify this claim by increasing the number of layers of our models from 2 to 10, which is the same as the range of layers mentioned in ACMP \cite{wang2022acmp}. In addition, similar to the work in ACMP \cite{wang2022acmp}, we compared our models with baselines of GCN \cite{kipf2016semi}, GAT\cite{velivckovic2018graph}, GRAND \cite{thorpe2022grand}, and ACMP-GCN \cite{wang2022acmp}. All experimental settings are consistent with the initial setups that are aforementioned. Furthermore, we note that in this experiment we only include homophilic graphs as \texttt{Cora, Citeseer, and Pubmed} since the smoothing effects induced from the baselines (i.e., GCN, GAT, and GRAND) are essentially not suitable for fitting the heterophilic graphs \cite{han2022generalized,han2023continuous}, which sharpening effects on node features are required. The learning accuracy is listed in Figure \ref{fig:osm}. One can check that all three baselines with considerable accuracy decrease via the increase of the layers, whereas both ACMP-GCN and UYGCN maintain relatively high learning accuracy with only slight decreases. This observation directly verifies the functionality of repulsive forces in terms of circumventing the over-smoothing issue.  

\begin{figure*}[ht]
    \centering
    \includegraphics[width = 0.99
    \textwidth, height = 0.4\textwidth]{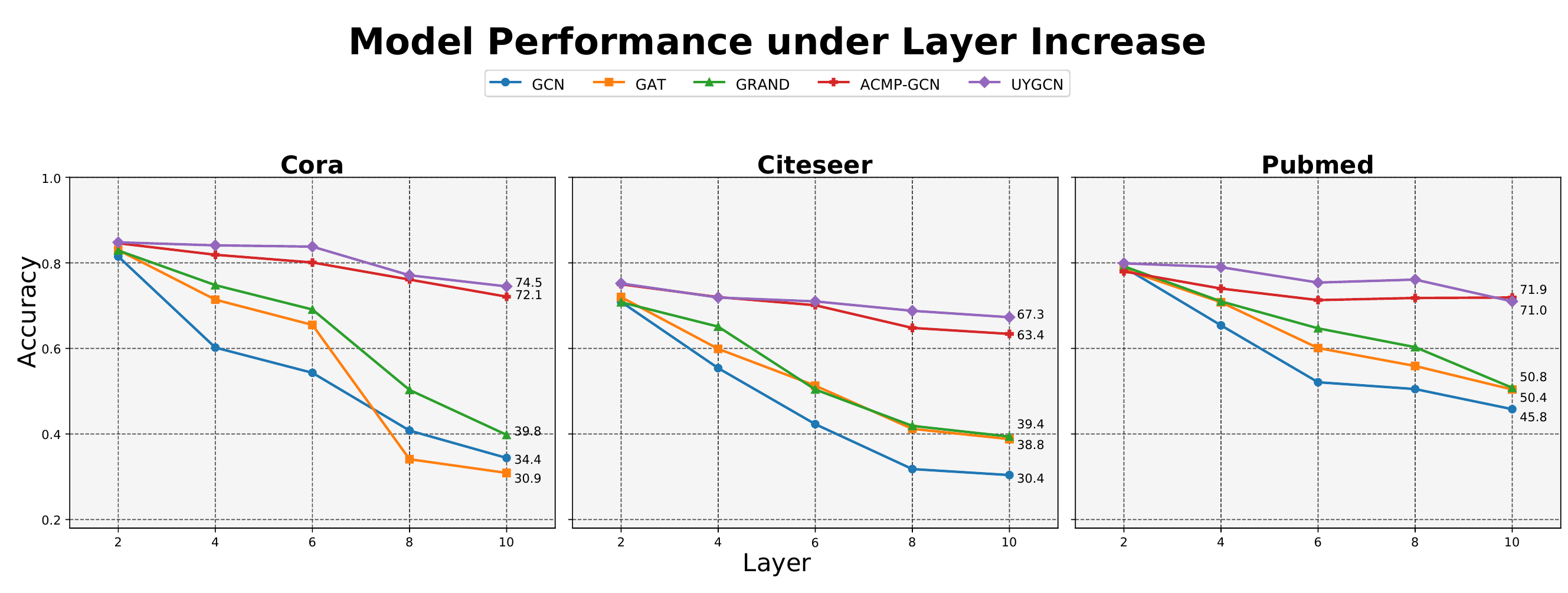}
    \caption{Model performance under the layer increases.}\label{fig:osm}
\end{figure*}

\subsection{Sensitivity Analysis}\label{sensitivity}
\begin{figure*}[t]
    \centering
    \includegraphics[width = 0.99
    \textwidth, height = 0.30\textwidth]{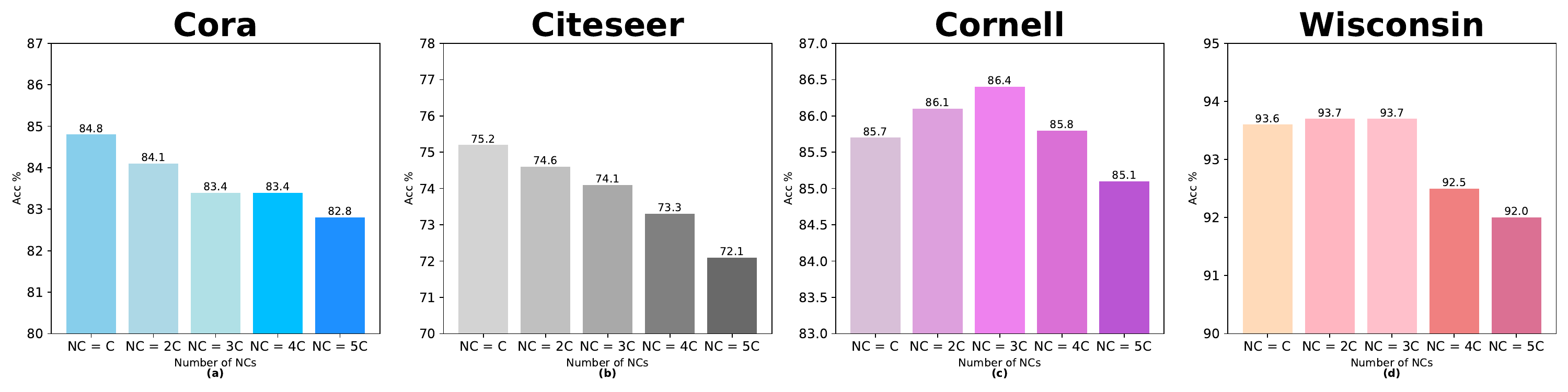}
    \caption{Learning Accuracy of UYGCN with different number of CNs.} \label{fig:sensitivity}
\end{figure*}
In this section, we conduct the sensitivity analysis of UYGCN on the number of CNs. Specifically, we aim to test whether a larger number of negatively weighted edges induced from more CNs will lead to higher repulsive forces between nodes. In general, adding more CNs to the graph will lead to a higher number of negatively weighted edges unless all labels are evenly distributed. To simplify our analysis, we only test the number of CNs as $C$, $2C$, $3C$, $4C$, and $5C$. We select \texttt{Cora}, \texttt{Citeseer}, \texttt{Cornell}, and \texttt{Wisconsin} and fix all other model parameters. The learning accuracy is presented in Figure \ref{fig:sensitivity}. One can observe that for homophilic graphs (\texttt{Cora} and \texttt{Citeseer}), there is a certain amount of accuracy drop with the increase of CNs. This directly verifies that a higher number of CNs introduce more repulsive forces from those newly added negatively weighted edges, making UYGCN less adaptable to homophilic graphs. On the other hand, our model accuracy is even with a slight increase when the number of CNs increases from $C$ to $3C$ via heterophilic graphs. This suggests that more repulsive forces should be preferred for fitting heterophilic graphs. However, with the number of CNs increasing to $4C$ and $5C$, the learning accuracy drops and yields an even worse outcome than $|CNs| = C$. This could be because when $|CNs| = 4C$ and $5C$, the power of double-well potential might not be sufficient to restrict the variations between node features. We highlight that, in this case, one may prefer to deploy a stronger trapping force to the system \cite{wang2022acmp}. 
Lastly, it is worth noting that a larger number of CNs might be more useful for a graph that contains a large number of nodes, while a small number of distinctive labels, as in this case, one shall prefer to have sufficient ``gravitational'' sources to induce appropriate attractive and repulsive forces to all nodes.

\subsection{Node Classification on Long range Graph Benchmarks.}\label{experiment_lrgb}

\begin{figure*}[h]
    \centering
    \includegraphics[scale=0.55]{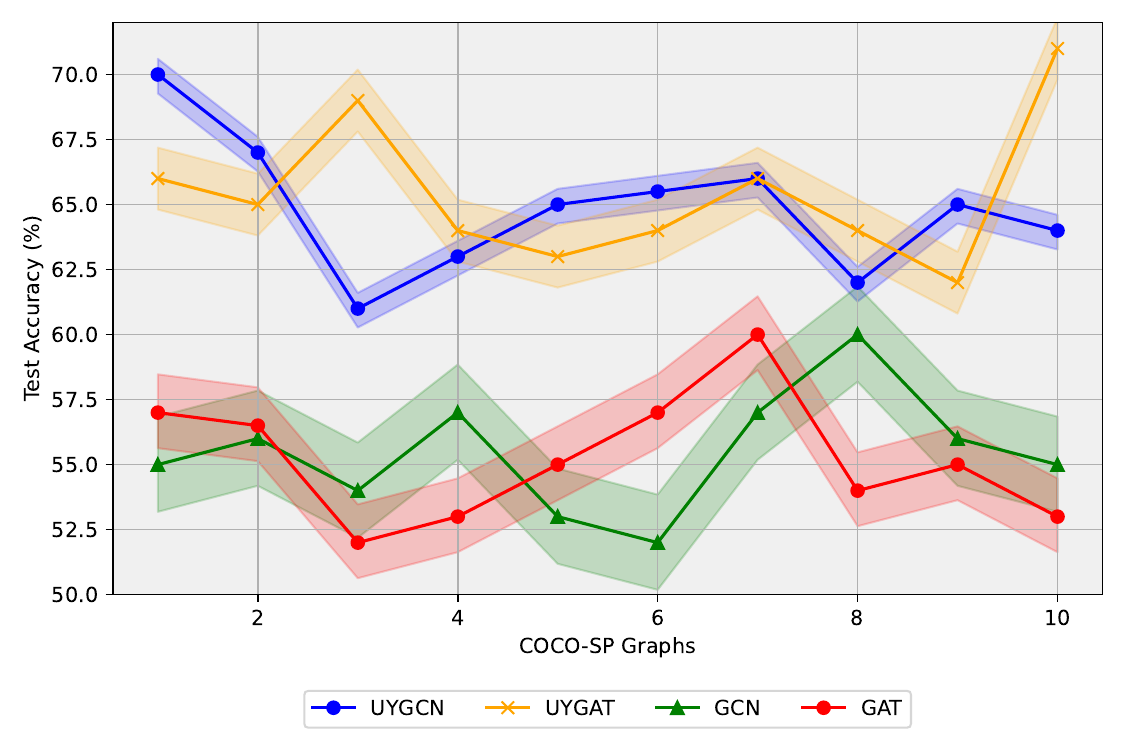}
    \caption{Learning Accuracy (F1 score) of UYGCN, UYGAT and their original models, GCN and GAT on long-term graph benchmarks.} \label{fig:coco_chat}
\end{figure*}

We test UYGCN and UYGAT over long-range graph benchmarks (LRGBs), namely \texttt{COCO-SP} provided in \cite{dwivedi2022long}. The dataset is designed to test whether one GNN model can capture long-range dependency between nodes under the metric of macro F1 score. Specifically, the \texttt{COCO-SP} dataset is a node classification dataset based on the MS COCO image dataset \cite{lin2014microsoft}
where each superpixel node denotes an image region belonging to a particular class. We highlight that, given the dataset contains over 100000 graphs \cite{dwivedi2022long}, our purpose for this analysis is to verify that UYGNNs can capture the long-term dependency between nodes in one graph. Therefore, we randomly sample 10 graphs from both datasets and report the average F1 scores {and standard derivations.} We compare UYGCN and UYGAT to the performance of GCN and GAT, and all models are 4 layers deep. The results are shown in Figure~\ref{fig:coco_chat}. One can find that both UYGCN and UYGAT show superior performance compared to their original counterparts, GCN and GAT, among all selected graphs. This suggests that our proposed approach can dramatically increase the model's power of capturing long-term dependency between nodes. {Lastly, we noticed that although, in general, attention-based models like GAT can outperform the original GCN via most classification tasks, when it comes to the tasks that require GNNs to capture the long-term relationship between nodes, UYGAT didn't show significant higher learning accuracy compared to UYGCN and similar observation can be found via the comparison between GCN and GAT. We leave the exploration of the effectiveness of the attention mechanism over long-range graph benchmarks in future work.}

\section{Related Works}\label{related_works}
\paragraph{Label Enhanced Approaches in GNNs}
Node labeling information has been leveraged in recent studies of GNNs to enhance their performance. For example, \cite{wang2021bag} and \cite{shi2020masked} utilized labeling information to enrich the node features, followed by the work \cite{chen2019highwaygraph} in which labeling information is leveraged for enhancing GNNs on the long-distance node relations (i.e., OSQ problem). \cite{yang2021self} further improved GNNs' performance using the labels from the outputs of the model. In addition, \cite{yang2019topology} deployed a topological optimization scheme to propagate labels under fixed points conditions \cite{shi2023fixed}, assuming that the nearby vertices in a graph tend to share the same label. Finally, the label-enhanced graph neural network (LEGNN) developed in \cite{yu2022label} expands the graph adjacency matrix with node labels and empirically shows the accuracy gain of the GNNs enhanced by their methods. 

\paragraph{Dirichlet Energy and OSM}
The so-called Dirichlet energy and its variants are the commonly applied measurement on the OSM issue of GNNs. Although the OSM issue has been observed for years \cite{cai2020note}, its commonly accepted definition has just been established recently \cite{rusch2023survey}. Nevertheless, many attempts have been made to mitigate the OSM issue via the lens of dynamic systems \cite{han2023continuous,di2022graph,thorpe2022grand}, multi-scale spectral filtering \cite{han2022generalized,shao2023unifying,yang2022quasi} as well as energy optimizations \cite{shao2022generalized,FuZhaoBian2022,zhai2023bregman}.

{
Specifically, \cite{di2022graph} shows that most of the GNN's propagations can be treated as the gradient flow of feature Dirichlet energy and its variants. If such energy is monotonically increasing via the propagation of one GNN, then such GNN is capable of handling the OSM problem as all nodes tend to have distinct features. In addition, \cite{di2022graph} also linked this observation to the characteristics of the filtering functions on the graph spectra and introduced the notion of so-called low/high-frequency dominance to asymptotically measure the OSM phenomenon. Following this work, \cite{han2022generalized} further generalized the traditional Dirichlet energy to the so-called graph framelet domain and showed that with the help of the low-pass and high-pass filtering functions, graph framelet can avoid the OSM problem and thus capable of handling both homophilic and heterophilic graphs. Finally, energy regularization approaches have also been applied. For example, \cite{FuZhaoBian2022,shao2022generalized} proposed $p$-Laplacian regularization to adjust the Dirichlet energy of the GNN outputs. Such methods yield the implicit layer scheme usually leveraged as an additional propagation layer to GNNs. 
}

\paragraph{Graph Topology and OSQ}
Unlike the OSM, the OSQ problem has just been identified and quantified recently \cite{alon2020bottleneck,topping2021understanding,shi2023exposition}, and little is known about a commonly acceptable definition of the OSQ issue. 
{Despite this, there are two general types of methods to mitigate OSQ, namely spatial and spectral graph rewiring methods, respectively. Although both spatial and spectral rewiring methods tend to rewire the graph to make the information flow easier to mitigate the OSQ problem, spatial rewiring methods target more on the graph's topological indicators, whereas spectral methods commonly focus on the rewiring based on graph spectral characteristics. 

Specifically, \cite{topping2021understanding} firstly linked the OSQ problem to the graph curvatures and showed that it is those edges with negative (balanced) Forman curvatures responsible for the OSQ problem, and a corresponding curvature flow-based rewiring approach is proposed. Later, the relationships between OSQ and other types of graph curvatures (i.e., Ricci curvature) are developed \cite{giraldo2022understanding,shi2023frameless, fesser2023mitigating}. On the other hand, spectral methods treat the OSQ problem as an ``unbalance'' of the identified spectral indicators such as spectral gap \cite{karhadkar2023fosr}, Cheeger constant \cite{banerjee2022oversquashing}, and effective resistance. It is worth noting that recent research \cite{giraldo2023trade,shao2023unifying} has also illustrated that although edges with negative curvatures are responsible for the OSQ problem, indicating a ``tough'' information transaction, those edges with very positive curvatures are, in general, responsible for the OSM problem since the node features are too ``easy'' to be homogenized by GNNs, suggesting a trade-off between OSQ and OSM. 

}

{
\section{Limitation of the Model}
Although this work introduced the CNs and the related labeling information-guided rewiring scheme to mitigate both OSM and OSQ problems, our model still has several limitations. For example, we are still unsure what the optimized number of CNs for a given graph is, although in most cases, setting the number of CNs equal to the number of classes is enough to make UYGNN generate state-of-the-art results. Secondly, even though we have shown that the rewiring process via UYGNN is able to increase the upper bound of the OSQ score, thus making UYGNN robust to the OSQ problem, in practice, whether UYGNN can actually achieve this upper bound is still immeasurable. In fact, to properly verify whether one GNN can significantly mitigate the OSQ problem is still an open question and a comprehensive framework in both theoretical and empirical aspects is still needed \cite{shi2023exposition}. Moreover, our experiment only shows the performance of UYGNN via node classification tasks; whether the labeling information remains powerful via other graph tasks such as graph pooling (classification) or anomaly detection is still unknown. Lastly, the effectiveness of labeling information-guided rewiring schemes leveraged in UYGNN is still unknown for graphs with multi-relations between nodes (i.e., Heterogeneous graphs and social networks) \cite{peng2024unsupervised} in which a more complex rewiring scheme, which contains both relation and the sign of edge weights sections, is required. We leave this to the future work.

}

\section{Concluding Remarks}
In this work, we developed a model agnostic approach to enhance GNNs to handle several major challenges via their propagation by leveraging the analogy of particle systems and node labeling information. We verified the properties of our model along with theoretical analysis and various empirical studies. We also investigated the functionality of negative eigenvalues in the graph spectrum from the perspective of heterophilic adaption and the additional edges induced from CNs through the lens of curvature. As CNs bring the training of GNNs in the realm of signed graphs, further exploration of the properties of signed graphs, such as the signed graph ``spectral gap'' and curvature for measuring the goodness of connectivity   
are needed \cite{belardo2019open} to develop more advanced GNN enhancement approaches. We leave these as future works.

\newpage

% \section{Formal Proofs}\label{appendx_proof}

\appendix

\section{Illustration on Cluster Flocking}\label{append: cluster_flocking}
The notion of cluster flocking is extensively studied in the field of particle systems and mathematical physics \cite{fang2019emergent,ahn2010stochastic,carrillo2010asymptotic}, serving as a tool for describing the evolution of a system. The famous Cucker-Smale system (CS) is one of the most popular systems that is considered a second-order system adopting classic dynamics. The recent work \cite{fang2019emergent} analyzed the potential bi-cluster flocking behaviors via the CS system with two ensembles with both repulsive and attractive forces that extend the conclusions from the previous studies where only attractive forces were considered. In the context of dynamic systems on graphs, the bi-cluster flocking can be defined in the following way,
\begin{defn}[Bi-cluster flocking \cite{wang2022acmp}]
    For a graph $\mathcal G$, its nodes are said to have a bi-cluster flocking behavior if there exist two disjoint sets of nodes subsets $\mathcal{V}_1=\{v^{(1)}_{i}\}^{N_1}_{i=1}$ and $\mathcal{V}_2=\{v^{(2)}_{i}\}^{N_2}_{i=1}$ ($N=N_1+N_2$) on which the dynamics $\{\mathbf h^{(1)}_{i}(t)\}_{i=1}^{N_1}$ and $\{\mathbf h^{(2)}_{i}(t)\}_{i=1}^{N_2}$  satisfy the following conditions 
    \begin{align}
        &\sup_{0\leq t \leq \infty} \max_ {1 \leq i,j \in  N_1} | h^{(1)}_{i}(t) -  h^{(1)}_{j}(t)| < \infty, \notag \\
        &\sup_{0\leq t \leq \infty} \max_ {1 \leq i,j \in  N_2} |h^{(2)}_{i}(t) -  h^{(2)}_{j}(t)| < \infty,
    \end{align}
    and 
    \begin{align}
        &\exists C', t^* >0, \text{ s.t. } \forall t > t^*, 
        \quad  \max_ {1 \leq i\in  N_1, 1 \leq j\in  N_2} \{| h^{(1)}_{i}(t) -  h^{(2)}_{j}(t) |\} \geq C', . 
    \end{align}
    where $ h^{(1)}_{i}(t)$ and $h^{(2)}_{j}(t)$ are any single component of vectors $\mathbf h^{(1)}_{i}(t)$ and $\mathbf h^{(2)}_{j}(t)$, respectively.
\end{defn}
 In the case of its discrete dynamic, the first condition in the above definition suggests given any iteration of the GNN defined on $\mathcal G$, the within-group node feature difference is bounded, whereas the second condition suggests there exists a specific number of layer $\ell^*$, after which the between-graph variation will always be greater than a constant $C'$.

Similar to ACMP, here we briefly illustrate that UYGAT can also achieve bi-cluster flocking. We highlight that given most of the proofs are in general similar to ACMP \cite{wang2022acmp} and the analysis in \cite{fang2019emergent}, we will mainly focus on the illustration of our conclusion. However, it is possible to observe bi-cluster flocking via UYGAT in which both attractive and repulsive forces are further reweighed from the attention coefficient, serving as a special type of ACMP in which $\beta_{i,j}$ is learnable rather than a positive constant. Thus according to Proposition 2 in \cite{wang2022acmp}, UYGAT has one bi-cluster flocking. 

Lastly, it is worth emphasizing again that although both our incoming results and ACMP model illustrate that our models can have node bi-cluster flocking, the necessary and sufficient conditions on how to achieve it are still unknown and worth future exploration.  

Let us turn to the UYGCN model with double-well potential with the dynamic as follows,
\begin{align}
    \frac{\partial}{\partial t}\mathbf h_i(t) = \sum_{j\in \mathcal N_i} a_c(\mathbf h_i, \mathbf h_j) (\mathbf h_j - \mathbf h_i) + \delta \mathbf h_i(1-\mathbf h^2_i), \label{Eq:14}
\end{align}
where for the convenience reason, we denote $\mathcal N_i$ as the neighbours of node $i$ in $\mathcal G_c$ which is the graph expanded by the CNs. We highlight that in ACMP, a strength coupling $(\alpha, \delta)$ is defined to further quantify the attractive and repulsive forces in the system. One can check that UYGCN is a special case of ACMP in terms of the strength coupling where $\alpha =1$. Next, we quantify bi-cluster flocking according to \cite{wang2022acmp} and \cite{fang2019emergent,carrillo2010asymptotic}. Further, due to the form of $\mathbf A_c$, the edge weights in Eq.~\eqref{Eq:14} are $\pm 1$ in front of their components difference. One can see that, based on the setting of UYGCN, the minimum attractive force ($+1$) is equal to the maximum repulsive force ($-1$). Therefore, the system is unlikely (but possible) to have bi-cluster flocking.

\section{Experiment Details}\label{append:experimental_details}
In this section, we provide additional details of our empirical analysis.

\subsection{Performance on Citation and Heterophilic Benchmarks}
We start with additional details on the performance of UYGNN and UYGAT via citation networks and heterophilic graphs. 

We tuned hyper-parameters using the grid search method for the parameter searching space. The search space for learning rate was in $\{0.1, 0.05, 0.01, 0.005\}$, number of hidden units in $\{16, 32, 64\}$, weight decay in $\{0.05, 0.01,0.005\}$, dropout in  $\{0.3, 0.5,0.7\}$, the coefficients $\delta$ for the double-well potential is initially set as $0$ for homophilic graphs and searched in the range of $\{0.5,1,1.5,2, 2.5\}$ for heterophilic graphs. 

Furthermore, we observed that doing normalization on $\mathbf A_c$ can improve UYGCN's performance in general. This could due to the fact that the normalization scheme $\mathbf D_c^{-\frac12} \mathbf A_c \mathbf D^{-\frac12}_c$ brings the degree information of the nodes to reweight $\mathbf A_c$, causing the imbalance between repulsive and attractive forces according to the definition of cluster flocking, and such imbalance may lead the system to cluster flocking. This provides the insights explaining the benefit of doing normalization on graph adjacency matrix even with negative weights. 

An additional interesting discovery in our empirical study is in some heterophilic graphs, i.e., \texttt{Cornell}, where we have a limited number of nodes. It is possible to have nodes in its training set obtained from the random split which don't cover all type of labels, i.e., total number of distinct labels in the training set less than $C$. In this case, the power of CNs will be largely diluted since there will be some CNs that do not induce any additional connections. In this case, we report the learning results of the average learning accuracy (10 times) on the training set that contains all types of labels. 

\subsection{LRGB Experiment}
We provide more descriptions on \texttt{COCO-SP} datasets and the implementation details on conducting our test on it. The \texttt{COCO-SP} dataset is a node classification dataset based on the MS COCO image dataset \cite{lin2014microsoft}
where each superpixel node denotes an image region belonging to a particular class. Based on \cite{dwivedi2022long}, there are 123,286 graphs with a total of 58.7 million nodes in \texttt{COCO-SP} where each corresponds to an image in MS COCO dataset. The graphs prepared after the superpixels
extraction have on average 476.88 nodes, with mean degree 5.65 and totally 332,091,912 edges with average edges of 2693.67 per graph. The evaluation metric for \texttt{COCO-SP} is the macro F1 score. 

Same as the reason mentioned in our main paper, since the main purpose for us to conduct UYGNNs on LRGB is to verify its advantage of capturing long-term relationships between nodes, we only randomly select 10 graphs from \texttt{COCO-SP} and report the learning accuracy of our models and their original counterparts. Different from the LRGB experiment conducted in some graph rewiring papers, for each selected, we maintained the split ratio as 20\% for training, 20\% for validation, and the rest for testing, thus resulting in a higher F1 score compared to \cite{dwivedi2022long,fesser2023mitigating}. Furthermore, to conduct a fair comparison, we set all hyper-parameters the same in the LRGB experiment.

\bibliographystyle{plain}
\bibliography{sn-bibliography}

\end{document}